\documentclass{article}

\usepackage{microtype}
\usepackage{graphicx}
\usepackage{subfigure}
\usepackage{booktabs} 
\usepackage[dvipsnames,table,xcdraw]{xcolor}
\usepackage{mathrsfs} 
\usepackage{bbold}
\usepackage{csquotes}
\usepackage{hyperref}
\usepackage{multirow}

\usepackage[accepted]{ICML_2025/icml2025}

\usepackage{amsmath}
\usepackage{amssymb}
\usepackage{mathtools}
\usepackage{amsthm}

\usepackage[capitalize,noabbrev]{cleveref}

\usepackage{pgfplots}
\pgfplotsset{compat=1.18}
\theoremstyle{plain}
\newtheorem{theorem}{Theorem}[section]
\newtheorem{proposition}[theorem]{Proposition}
\newtheorem{lemma}[theorem]{Lemma}
\newtheorem{corollary}[theorem]{Corollary}
\theoremstyle{definition}
\newtheorem{definition}[theorem]{Definition}

\theoremstyle{remark}

\newif\ifdraft
\draftfalse
\drafttrue

\ifdraft
 \newcommand{\pf}[1]{{\color{red}{#1}}}
 \newcommand{\PF}[1]{{\color{red}{{\bf PF}: \bf #1}}}
 \newcommand{\ef}[1]{{\color{blue}{#1}}}
 \newcommand{\EF}[1]{{\color{blue}{{\bf EF: #1}}}}
 
 \newcommand{\CH}[1]{{\color{cyan}{pb: #1}}}
 
 \newcommand{\AS}[1]{{\color{cyan}{AS: #1}}}
 \newcommand{\TODO}[1]{{\textcolor{red} {TODO: #1}}}
 \newcommand{\ONGOING}[1]{{\textcolor{orange}  {ONGOING: #1}}}

\else
 \newcommand{\pf}[1]{#1}
 \newcommand{\ef}[1]{#1}

 \newcommand{\PF}[1]{}
 \newcommand{\EF}[1]{}
 \newcommand{\CH}[1]{}
  \newcommand{\AS}[1]{}
  \newcommand{\TODO}[1]{}
  \newcommand{\ONGOING}[1]{{}
\fi

\renewcommand{\comment}[1]{}
\newcommand{\parag}[1]{\vspace{-0.9mm}\paragraph{#1}}
\newcommand{\sparag}[1]{\vspace{-0.9mm}\subparagraph{#1}}
\newcommand{\subsec}[1]{\vspace{-0.4mm}\subsection{#1}}

\newcommand{\figy}[2]{\includegraphics[height=#1,keepaspectratio]{#2}}
\newcommand{\figx}[2]{\includegraphics[width=#1,keepaspectratio]{#2}}

\DeclareMathOperator*{\argmin}{arg\,min}

\newcommand{\mL}[0]{\mathcal{L}}
\newcommand{\mM}[0]{\mathcal{M}}
\newcommand{\mR}[0]{\mathcal{R}}
\newcommand{\mS}[0]{\mathcal{S}}
\newcommand{\mZ}[0]{\mathcal{Z}}

\newcommand{\bA}[0]{\mathbf{A}}
\newcommand{\bc}[0]{\mathbf{c}}
\newcommand{\bl}[0]{\mathbf{\Lambda}}
\newcommand{\bt}[0]{\mathbf{\Theta}}
\newcommand{\bC}[0]{\mathbf{C}}
\newcommand{\bI}[0]{\mathbf{I}}
\newcommand{\bQ}[0]{\mathbf{Q}}
\newcommand{\bR}[0]{\mathbf{R}}
\newcommand{\bu}[0]{\mathbf{u}}
\newcommand{\bU}[0]{\mathbf{U}}
\newcommand{\bm}[0]{\mathbf{m}}
\newcommand{\bp}[0]{\mathbf{p}}
\newcommand{\bv}[0]{\mathbf{v}}
\newcommand{\bX}[0]{\mathbf{X}}
\newcommand{\bx}[0]{\mathbf{x}}
\newcommand{\by}[0]{\mathbf{y}}
\newcommand{\bY}[0]{\mathbf{Y}}
\newcommand{\bw}[0]{\mathbf{w}}
\newcommand{\bW}[0]{\mathbf{W}}
\newcommand{\bz}[0]{\mathbf{z}}
\newcommand{\bZ}[0]{\mathbf{Z}}

\newcommand{\ent}{\mathbb{H}}
\newcommand{\E}{\mathbb{E}}
\newcommand{\C}{\mathbb{C}}
\newcommand{\var}{\mathbb{V}}
\newcommand{\x}{\mathbf{x}}
\newcommand{\weight}{\boldsymbol{\theta}}
\newcommand{\bs}{\mathbf}
\renewcommand{\deg}{\mathsf{deg}}
\newcommand{\inspace}{\mathbb R^{c_\text{in}}}
\newcommand{\outspace}{\mathbb R^{c_\text{out}}}
\newcommand{\R}{\mathbb{R}}
\renewcommand{\d}{\mathrm{d}}
\newcommand{\cov}{\mathrm{cov}}
\newcommand{\prob}{\mathbb{P}}
\newcommand{\inv}{^{-1}}

\renewcommand{\mathsf}[1]{\textsf{#1}} 
\hypersetup{
    colorlinks=true,
    linkcolor=RoyalBlue,
    filecolor=RoyalBlue,      
    urlcolor=RoyalBlue,
    citecolor=RoyalBlue,
    pdftitle={Bibliographic Report},
    pdfpagemode=FullScreen,
    }

\usepackage[textsize=tiny]{todonotes}

\icmltitlerunning{Do you understand epistemic uncertainty? Think \textit{again}!}

\begin{document}

\twocolumn[
\icmltitle{Do You Understand Epistemic Uncertainty? Think \textit{Again}!\\Rigorous Frequentist Epistemic Uncertainty Estimation in Regression}



\icmlsetsymbol{equal}{*}

\begin{icmlauthorlist}
\icmlauthor{Enrico Foglia}{isae,sher}
\icmlauthor{Benjamin Bobbia}{isae}
\icmlauthor{Nikita Durasov}{epfl}
\icmlauthor{Michael Bauerheim}{isae}
\icmlauthor{Pascal Fua}{epfl}
\icmlauthor{Stephane Moreau}{sher}
\icmlauthor{Thierry Jardin}{isae}

\end{icmlauthorlist}

\icmlaffiliation{isae}{Department of Aerodynamics, Energetics and Propulsion,, Institut Sup\'erieur de l’A\'eronautique et de l’Espace, Toulouse, France}
\icmlaffiliation{epfl}{Computer Vision Laboratory, Ecole Polyt\'echnique Fed\'erale de Lausanne, Lausanne, Switzerland}
\icmlaffiliation{sher}{Department of Mechanical Engineering, Universit\'e de Sherbrooke, Sherbrooke, Canada}

\icmlcorrespondingauthor{Enrico Foglia}{enrico.foglia@isae-supaero.fr}

\icmlkeywords{Machine Learning, ICML}

\vskip 0.3in]



\printAffiliationsAndNotice{\icmlEqualContribution} 

\begin{abstract}
Quantifying model uncertainty is critical for understanding prediction reliability, yet distinguishing between aleatoric and epistemic uncertainty remains challenging. We extend recent work from classification to regression to provide a novel frequentist approach to epistemic and aleatoric uncertainty estimation. We train models to generate conditional predictions by feeding their initial output back as an additional input. This method allows for a rigorous measurement of model uncertainty by observing how prediction responses change when conditioned on the model's previous answer. 
We provide a complete theoretical framework to analyze epistemic uncertainty in regression in a frequentist way, and explain how it can be exploited in practice to gauge a model's uncertainty, with minimal changes to the original architecture. 
\end{abstract}

\section{Introduction}\label{sec:intro}



\begin{figure}
    \centering
    \input{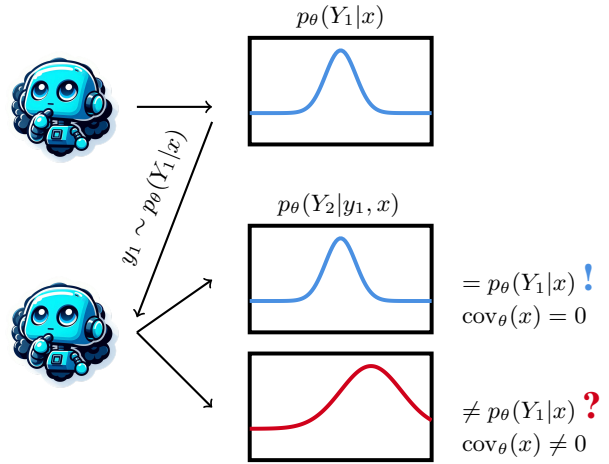}
    \caption{\small {\bf Estimating epistemic uncertainty.} (a) The model is run twice, the first time normally and the second time with the first prediction as a further input. The covariance of the two outputs can be used to quantify the epistemic uncertainty. Heuristically, it can be said that a model that double guesses its own answers presents some degree of epistemic uncertainty.}
    \label{fig:teaser}
\end{figure}

Prediction errors have two main causes. The first one is the stochasticity inherent to the data used as input (for example measurement noise, ambiguous labeling, data issued of a truly random process) and is referred to as \emph{aleatoric} uncertainty. The second is potential inaccuracies in the model used to make the predictions and is referred to as \emph{epistemic} uncertainty. These two are always present but, crucially, epistemic uncertainty can be reduced by gathering more training data.  Thus, being able to separate aleatoric and epistemic uncertainty is key to knowing when more data needs to be collected. 

Unfortunately,  recent work~\cite{bengs2023second} has shed doubts about the possibility of training models to faithfully estimate their own epistemic uncertainty, at least in a frequentist manner. However, a workaround has been proposed~\cite{johnson2024experts}, in a classification setting: if a model can be trained to give two potentially correlated responses $y_1$ and $y_2$ for every input $ x$, then a rigorous measure of epistemic uncertainty can be constructed in a frequentist manner. Practically, this can be achieved by first running the model normally and then repeating the process by adding the model answer to the input, a techinique that had already been proposed in \cite{Durasov24a}. Fig.~\ref{fig:teaser} illustrates this idea. The intuition behind this method is that a confident model will not double-guess its own answers given the new inputs, while the presence of epistemic uncertainty may induce the model to ``change its mind'' and return a different answer the second time. Thus, how much the answers change can be used as a measure of epistemic uncertainty. The main objective of our paper si to propose a general framework unifying and generalizing these approaches.

Even though it is impossible to train a model to report its epistemic uncertainty without making assumptions on the data distribution,  this roadblock can be avoided when one can construct a dataset composed of triplets $( x,y_1,y_2)$. It is important to make sure that for every input $x$, $y_1$,$y_2$ are two measurements independently sampled from the distribution $p(y| x)$. Intuitively, in this way something is now known for sure about the data distribution: it can be decomposed as $p(y_1,y_2\vert x)=p(y_1\vert x)\cdot p(y_2\vert x)$. In a recent paper, \cite{johnson2024experts} showed that this is enough to correctly gauge the epistemic state of the learner, but the scope was limited to classification. In fact, their approach cannot handle a regression problems, where outputs are real numbers, because the output space $\mathcal Y$ is continuous.

Training using more than one output per input is not standard practice in the deep learning community. However, in experimental sciences, it is common to repeat experiments more than once, or to collect long time signals from the sensors, to be able to estimate error bounds. Thus, the envisioned scenario is important in all scientific fields  where experiments can be repeated.
The following advances are proposed:  
\begin{itemize}

    \item The approach of~\cite{johnson2024experts} is extended to regression. The proposed approach is general, in the sense that it does not make hypotheses on the form of the predictive distribution, while being easy to implement as it requiring minimal changes to the model architecture.
    
    \item In concurrent research \cite{Durasov24a}, the idea of estimating the uncertainty of a model by running it once and then feeding it back its first answer as an additional input has been demonstrated with empirical success but little theoretical justification. The mathematical developments introduced in this work provide a formal grounding for this feedback-based approach, while also highlighting some of its current limitations. 

\end{itemize}

The effectiveness of our approach is demonstrated both on synthetic and experimental data (wind tunnel and anecho\"ic room measurements).  The code will be made available upon acceptance of the manuscript.

\section{Methodology}\label{sec:math}



Rigorously, epistemic uncertainty should capture the distance between the predictive distribution $p_\theta$ and the data distribution $p$. Thus, it should be formalized as a probability distribution \emph{over the space of probability distributions}. However, computing useful confidence intervals without information about the underlying distribution is impossible \cite{low1997nonparametric}, and no loss exists that incentivize the model to put forward a reliable estimation of its internal uncertainty \cite{bengs2023second}. 

In fact, the best that one can hope to achieve in the most general setting is a calibrated model:
\begin{equation}
    p_\theta(y\vert  x) \triangleq \E_{X\sim p(X\vert [ x])}[p(y\vert X)]\;,
\end{equation}
where $[ x]$ is the equivalence class of all points that the model cannot distinguish. Such models can give a reliable information about the total uncertainty, but are unable to separate it into its aleatoric and epistemic components. 

To overcome this difficulty, \cite{johnson2024experts} propose to sample the data distribution twice for each input, making sure the sample are independent. This way, the true data distribution can be factored as $p(y_1,y_2\vert  x)=p(y_1\vert  x)\cdot p(y_2\vert  x)$. This is enough to give an estimation of the epistemic uncertainty of a model trained to predict pairs $p_\theta(y_1,y_2\vert  x)$, since any correlation between the outputs (for a given $X= x$) can only be attributed to a modeling error. This suggests to use the model covariance as a measure of the epistemic uncertainty. In particular, it will be proven that:
\begin{equation}
\begin{aligned}
    \cov_\theta( x) &\triangleq \C_{Y_1,Y_2\sim p_\theta(Y_1,Y_2\vert  x)}[Y_1, Y_2]\\
    &=\var_{X\sim p(X\vert[ x])}[\E_{Y\sim p(Y\vert X)}[Y]]\;.
\end{aligned}
\end{equation}
This implies that the model covariance gives a measure of the grouping loss; i.e. the epistemic error arising from lumping together points into $[ x ]$ which should instead be distinguished. Incidentally, grouping loss is the only epistemic error present in a perfectly calibrated model.
\subsection{Formalization}
\label{sec:formal}

Let $(X,Y)\in \mathcal X\times \mathcal Y$ be random variables with joint distribution $p(X,Y)$, which we will refer to as the input and output, respectively.  Typically, $X$ characterizes the state of a physical system while $Y$ represents how the system performs while in that state and it is assumed that there is a functional relationship between one and the other. Thus, the main quantity of interest is the conditional distribution $p(Y\vert X= x)$, which represents how much uncertainty on the value of $Y$ remains after observing a specific value $ x$ of $X$. 

To fully capture this uncertainty, the functional relationship between $X$ and $Y$ must be modeled using a full probability density over $\mathcal Y$, rather than a single value.  To approximate the true posterior probability $p(Y\vert X= x)$ , a probabilistic model $p_\theta(Y\vert X= x):\mathcal X\to \Delta_\mathcal Y$ is introduced, where $\Delta_\mathcal Y$ is the space of probability density functions over $\mathcal Y$. In practice, $p_\theta$ is typically implemented using a deep network with weights  $\theta$, learned as discussed below.

Throughout the paper, expectation operators taken with respect to the predicted distribution will be denoted by $\E_{Y\sim p_\theta(Y\vert  x)}[Y]\triangleq \E_\theta[Y\vert  x]$, for notational simplicity. Similarly, expectation with respect to the data distribution will be denoted as $\E_{Y\sim p(Y\vert x)}[Y]\triangleq\E[Y\vert x]$. The same will be true for variance and covariance operators $\var$ and $\C$. The proofs of theorems stated in the remainder of this section are given in appendix~\ref{sec:proofs}.

\subsection{Calibration}
\label{sec:calibration}

Calibration is one of the main metrics used to evaluate the quality of a probabilistic model, such as $p_{\theta}$. Informally speaking, a calibrated model produces the correct distribution on average. In this context, the average is taken over all inputs that the model cannot distinguish: within this set, the model is allowed to make mistakes provided they end up canceling out at the end. A more  rigorous and widely used description is:
\begin{definition}\label{def:calibration}
    Let $[ x]$ be the equivalence class $\{ x'\vert\; p_\theta(y\vert  x')=p_\theta(y\vert x),\;\forall y\in\mathcal{Y}\}$. A model $p_\theta(y\vert  x)$ is said to be first-order distribution calibrated if: 
    \begin{align}\label{eqn:first order calibration}
        p_\theta(y\vert  x) & = \E_{p(X\vert X\in [ x])}[p(y\vert X)] \; , \nonumber \\
      & = \E[p(y\vert X)\vert [ x]] \; .
    \end{align}
\end{definition}
For all theoretical derivations, models will be assumed to be calibrated.
First-order calibration is achievable either by training on a large enough dataset or by post-hoc recalibration \cite{song2019distribution}, as long as a calibration set has been separated from the training and testing datasets. 


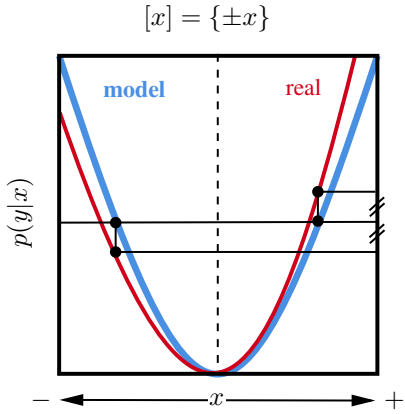
\begin{figure}[!hbt]
\begin{center}
        \tikzset{every picture/.style={line width=0.75pt}} 

\begin{tikzpicture}[x=0.75pt,y=0.75pt,yscale=-1,xscale=1]

\draw  [dashed]  (339.5,45) -- (339.5,206) ;
\draw [color={rgb, 255:red, 74; green, 144; blue, 226 }  ,draw opacity=1 ][line width=2.5]    (259.5,46) .. controls (328.5,258.5) and (349.5,262) .. (420,46) ;
\draw [color={rgb, 255:red, 208; green, 2; blue, 27 }  ,draw opacity=1][line width=1.5]    (259.8,74.6) .. controls (331.8,284.2) and (365.8,219.8) .. (408.8,45.8) ;
\draw  [line width=1.5]  (259.5,46) -- (420,46) -- (420,206.5) -- (259.5,206.5) -- cycle ;
\draw    (261.67,219.67) -- (417,220) ;
\draw [shift={(419,220)}, rotate = 180.12] [fill={rgb, 255:red, 0; green, 0; blue, 0 }  ][line width=0.08]  [draw opacity=0] (12,-3) -- (0,0) -- (12,3) -- cycle    ;
\draw [shift={(259.67,219.67)}, rotate = 0.12] [fill={rgb, 255:red, 0; green, 0; blue, 0 }  ][line width=0.08]  [draw opacity=0] (12,-3) -- (0,0) -- (12,3) -- cycle    ;
\draw  [draw opacity=0][fill={rgb, 255:red, 255; green, 255; blue, 255 }  ,fill opacity=1 ] (334.83,214) -- (343.83,214) -- (343.83,225.67) -- (334.83,225.67) -- cycle ;
\draw    (259.4,130.2) -- (420.2,129.8) ;
\draw    (288,130.2) -- (288,145) ;
\draw    (390,114.6) -- (390,129.4) ;
\draw    (288.6,145) -- (420.2,145) ;
\draw    (390.2,114.6) -- (419,114.6) ;
\filldraw  (288,130.2) circle (1.8pt);
\filldraw  (390,114.6) circle (1.8pt);
\filldraw  (288,145) circle (1.8pt);
\filldraw  (390,129.4) circle (1.8pt);

\draw    (416,137.6) -- (423.4,131) ;
\draw    (416.2,140.6) -- (423.6,134) ;
\draw    (416,124) -- (423.4,117.4) ;
\draw    (416.2,127) -- (423.6,120.4) ;

\draw (333.67,215) node [anchor=north west][inner sep=0.75pt]    {$x$};
\draw (422,214) node [anchor=north west][inner sep=0.75pt]    {$+$};
\draw (243.67,214) node [anchor=north west][inner sep=0.75pt]    {$-$};
\draw (231.73,144) node [anchor=north west][inner sep=0.75pt]  [rotate=-270]  {$p( y|x)$};
\draw (277,57) node [anchor=north west][inner sep=0.75pt]   [align=left] {{ \textcolor[rgb]{0.29,0.56,0.89}{\small \textbf{model}}}};
\draw (372.33,56) node [anchor=north west][inner sep=0.75pt]   [align=left] {{\textcolor[rgb]{0.82,0.01,0.11}{\small real}}};
\draw (300.6,18.8) node [anchor=north west][inner sep=0.75pt]    {$[x] =\{\pm x\}$};

\end{tikzpicture}
        \vspace{-2cm}
\end{center}
\caption{\small {\bf A  calibrated model can make mistakes.}  In this example, the model is symmetric around the ordinate and cannot distinguish positive from negative inputs. Even if the target distribution is asymmetric, the model can be calibrated because  errors on  the opposite sides of the $y$-axis cancel out. Since these errors are due to the model and can in principle be reduced by gathering more data, they are of an epistemic nature.
}
\label{fig:calibration}
\end{figure}

Note that a calibrated model can still make mistakes within equivalence classes, which aggregate  the points that the model \emph{cannot distinguish}. Fig.~\ref{fig:calibration} provides an example of this behavior. This observation is formalized with the following theorem:
\begin{theorem}\label{tmh: model moments}
    Let $p_\theta(y\vert  x)$ be a first-order calibrated model. Then:
        \begin{align}
            \mu_\theta( x) &\triangleq \E_{\theta}[Y\vert x] = \E[\E[Y\vert X]\vert [ x]]\\
            \sigma_\theta^2( x) &\triangleq \var_{\theta}[Y\vert x] = \E[\var[Y\vert X]\vert [ x]]+\var[\E[Y\vert X]\vert [ x]]\;. \nonumber
        \end{align}
\end{theorem}
Thus, while the mean of a calibrated model is equal to the mean over the equivalence class $[ x]$ of the true means, the variance is the sum of the mean of the true variances and the variance of the true means. The first term is the aleatoric part that pertains to the ground truth distribution, while the second is of an epistemic nature because it stems from lumping together points that should have remained separated. 

This  refines earlier statements found in the literature that ``if a calibrated model predicts a distribution with some mean $\mu$ and variance $\sigma^2$, then it means that on average over all cases with the same prediction the mean of the target is $\mu$ and variance is $\sigma^2$''~\cite{song2019distribution}, which essentially ignores the term  $\var[\E[Y\vert X]\vert[x]]$. 
Since recalibration methods take a variance prediction $\tilde\sigma_\theta$ that is not calibrated and map it to a $\sigma_\theta=s(\tilde\sigma_\theta)$ that obeys Def.~\ref{def:calibration}, this formal imprecision has no effect on recalibration procedures.

\subsection{Epistemic Uncertainty}
\label{sec:epistemic}

To separate the total uncertainty into its aleatoric and epistemic components, models trained to predict pairs are considered. Such models shall be fitted using datasets composed of triplets $( x, y_1, y_2)$, where $y_1$ and $y_2$ are sampled iid from $p(Y\vert  x)$. Because of this, a model that is first order calibrated at predicting pairs will be of the form:
\begin{definition}\label{eqn: first order calibration pairs}
$p_\theta(y_1,y_2\vert  x)$ is first-order calibrated at predicting pairs if:
\begin{equation}
\begin{aligned}
        p_\theta(y_1,y_2\vert x) &= \E[p(y_1,y_2\vert X)\vert [ x]] \; ,\\
        &=\E[p(y_1\vert X)\cdot p(y_2\vert X)\vert [ x]] \; .
\end{aligned}
\end{equation}
\end{definition}
Note that training on pairs does not deteriorate the performance on single-output predictions, since the marginal distribution $p_\theta(y_1\vert  x)$ will remain first-order calibrated.
\begin{theorem}\label{thm:marginal calibration}
    Let $p_\theta(y_1,y_2\vert x)$ be first-order calibrated at predicting pairs. Then its marginals $p_\theta(y_1\vert x)$ and $p_\theta(y_2\vert x)$ are first order calibrated over $p(y\vert x)$.
\end{theorem}
In particular, the variance of the marginal distribution will have the same decomposition as in Thm.~\ref{tmh: model moments}. The advantage of using pairs of outputs is that now the predicted correlation between the two answers can be used as a measure of the epistemic uncertainty. More precisely:
\begin{theorem}\label{thm: model covariance}
    Let $p_\theta(y_1,y_2\vert  x)$ be first-order calibrated at predicting pairs. Then:
    \begin{equation}
        \cov_\theta( x)\triangleq\mathbb C_\theta[Y_1,Y_2\vert  x] = \var[\E[Y_1\vert X]\vert [ x]] \; .
    \end{equation}
\end{theorem}
This result is formally very similar to Thm.~\ref{tmh: model moments}, with the key difference that now, by construction, $\E[\C[Y_1,Y_2\vert X]\vert[ x]]=0$.

The epistemic uncertainty estimate in Thm.~\ref{thm: model covariance} depends on our ability to sample the distribution $p(y \vert  x)$ twice independently for each $ x$ to produce triplets of the form $( x,y_1,y_2)$. Yet, the vast majority of datasets only contain pairs of the form $( x,y)$. In this case, the model can be trained using triplets of the form $( x,y,y)$. Then, it is possible to prove the following. 

\begin{theorem}\label{thm:zigzag final theorem}
    Suppose to train a model to predict pairs, but drawing only one sample from the data distribution. In this case $p(y_1,y_2\vert x) = p(y_1\vert x)\delta(y_2-y_1) $ , making the two samples perfectly correlated. Then:
    \begin{equation}
        \cov_\theta(x) = \E[\var[Y_1\vert X]\vert[x]] + \var[\E[Y_1\vert X]\vert [x]] \; .
    \end{equation}
\end{theorem}
Thus, a model trained to predict pairs trained on couples $( x, y)$ of data instead of triplets $( x,y_1,y_2)$ can only estimate the total uncertainty, but not separate the aleatoric from the epistemic. This confirms the impossibility to train a model to report its epistemic uncertainty without making any sort of hypothesis on the data generating distribution.

Computing the model covariance is straightforward. First note  that the joint distribution can be written as:
\begin{equation}\label{eqn:decomposition joint zigzag}
    p_\theta(y_1,y_2\vert x) = p_\theta(y_2\vert y_1, x)\cdot p_\theta(y_1\vert x) \; ,
\end{equation}
where the first term on the right-hand-side can be computed by feeding back its own answers to the model.  It can be shown that:

\begin{equation}\label{eqn:cov approximation}
    \cov_\theta( x) \!\simeq\! \frac{1}{M}\sum_{m=1}^M y_m\mu_\theta( x\vert y_m)- \mu_\theta^2( x),  \; y_m\sim p_\theta(y\vert  x) \;.
\end{equation}
where, since the sum is a Monte Carlo approximation of an integral, the equality is asymptotically exact for $M\to\infty$. Thus, if $p_\theta(y_2\vert y_1, x)=p_\theta(y_1\vert  x)$, then the epistemic variance is zero. In contrast, if the model second-guesses its own answers, then $p_\theta(y_2\vert y_1, x)\neq p_\theta(y_1\vert  x)$, resulting a non-zero epistemic uncertainty. In theory, the covariance can  take positive or negative values. In practice though, one is only interested in the absolute value of $\cov_\theta$, which can be used as an indicator of epistemic uncertainty.

\subsection{Training the Model}\label{sec:training}


In the classification setting, the softmax activation in the last layer of a neural network makes every classifier a probabilistic model. Since the output space $\mathcal Y$ is finite, a model of this kind can potentially produce any distribution in $\Delta_\mathcal Y$. In regression, since $\mathcal Y$ is continuous, it is not possible to predict the probability density function $p_\theta$ in full generality. Instead, one is forced to use models in a subset $\widehat{\Delta}_\mathcal Y\subset\Delta_\mathcal Y$, where the densities can be given in terms of a finite set of parameters. The most common choice is to use Gaussian distributions $p_\theta(y\vert x) =\mathcal N(y\vert\mu_\theta( x), \sigma^2_\theta( x))$, and let the output of the model be directly the mean $\mu_\theta$ and the variance $\sigma_\theta^2\geq 0$. Notice, however, that the discussion presented in this paper is not limited to this choice: the only requirement is to be able to compute $\E_\theta$ and $\var_\theta$ given any particular choice of parametric distribution.

Once the particular form of $p_\theta$ has been chosen, training a model to predict couples is not more difficult than training in the usual fashion. Indeed, the decomposition of Eq.~(\ref{eqn:decomposition joint zigzag}) allows training $p_{\theta}$ by minimizing the negative log likelihood:
\begin{align}
    \ell_\text{NLL}(y_1,y_2,p_\theta(\cdot, \cdot\vert x)) &= -\log p_\theta(y_1,y_2\vert x) \; , \\
                                                                                     &= -\log p_\theta(y_1\vert x) \! - \! \log p_\theta(y_2\vert y_1,x)   , \nonumber
\end{align}
where the first term is the first prediction and the second is the output when concatenating the ground truth to the input. This methodology is appealing because:
\begin{itemize}
    \item It requires training only one model, with minimal changes to the base architecture and the training procedure. 
    \item Unlike in MC Dropout~\cite{gal2016dropout} and similar sampling-based methods, the sampling step of the present methodology can be performed in parallel by batching all the samples $y_m$, since the weights are treated deterministically.
\end{itemize}
%

\section{Related Work}\label{sec:sota}

\subsection{Bayesian Deep Learning}

Bayesian Deep Learning is the most successful framework today to predict and analyze epistemic uncertainty in neural networks. 
The main object studied within the Bayesian framework is the so-called weight posterior:
\begin{equation}
    p(\theta\vert \mathcal D)\propto p(\theta)p(\mathcal D\vert \theta) \; ,
\end{equation}
where $p(\theta)$ is the prior on the weights $\theta$ and $p(\mathcal D\vert \theta)$ is the likelihood, i.e. the probability to observe the dataset $\mathcal D$ if the weights are set to $\theta$. The predicted distribution after training, $ p(y\vert x,\mathcal D)$, can then be given by averaging over the possible parameters as:
\begin{equation}\label{eqn:model averaging}
    p(y\vert x,\mathcal D)=\int p_\theta(y\vert x)p(\theta\vert \mathcal D)\d \theta\;,
\end{equation}
where $p_\theta(y\vert x)$ is a specific instance of the model with weights equal to a specific value of $\theta$. This leads to a variance decomposition, using the law of total variance, which is very similar to that in Thm.~\ref{tmh: model moments}:
\begin{equation}
    \var[Y\vert x,\mathcal{D}] = \underbrace{\E[\var_\theta[Y\vert x]\vert \mathcal D]}_\text{aleatoric}+\underbrace{\var[\E_\theta[Y\vert x]\vert \mathcal D]}_\text{epistemic}\;.
\end{equation}
The computational intractability related to the accurate calculation of $p(\theta\vert \mathcal{D})$ and the averaging in Eq.~(\ref{eqn:model averaging}), has led to several approximation techniques, of which the most popular and effective is Deep Ensembles (DE) \cite{lakshminarayanan2017simple, wild2024rigorous}. However, DE relies on training multiple copies of the same network, which can be extremely costly in terms of training time and memory requirement. Other techniques exist to make Bayesian modeling more accessible, but, as a general rule of thumb, the cheaper the approximation the worst the estimation accuracy. Methods like Deep Ensembles or Hamiltonian Monte Carlo \cite{izmailov2021bayesian, neal2011mcmc} sit on the expensive-accurate side of the spectrum, whereas variational inference methods like Monte Carlo Dropout \cite{gal2016dropout} (and variants of this method) or the Laplace approximation \cite{mackay1992bayesian} are easier to compute, but less precise. 

\subsection{Metrics and Calibration}

The definition of calibration of Def.~\ref{def:calibration} has been extensively used in the classification community, and has been introduced in the context of regression by \cite{song2019distribution}. Other definitions of calibration include quantile calibration and variance calibration \cite{levi2022evaluating}. While these definitions are more practical to check, the notion of distribution calibration is more suited for theoretical purposes. Also, training a model by minimizing a proper scoring loss (like the negative log likelihood) will eventually yield calibrated models \cite{gneiting2007strictly}. It can be proven that, by using a slightly stronger notion of variance calibration, Thm.~\ref{tmh: model moments} and Thm.~\ref{thm: model covariance} still hold. The discussion about such modification is deferred to appendix~\ref{sec:variance calibration}. 

\subsection{Connection to~\cite{johnson2024experts}} 

In the case of classification over $K$ classes, the output probability space $\Delta_\mathcal{Y}$ is just the $K-1$ dimensional probability simplex, so that $p_\theta(y\vert x)$ is just a vector in $\mathbb R^{K}$ the components of which are non-negative and sum up to one. Similarly, the joint probability $p_\theta(y_1,y_2\vert x)$ can be interpreted as a stochastic $K\times K$ matrix. 
To exploit this, \cite{johnson2024experts} define a {\it covariance operator} to measure the difference between the model, which may include correlations, and the expected ground truth where the outputs are uncorrelated.
\begin{definition}\label{eqn:covariance operator}
Let $p_\theta(y_1,y_2\vert x)$ be a model trained to predict pairs such that $p_\theta(y_1\vert x)$ and $p_\theta(y_2\vert x)$ are its marginal distributions. Define the covariance operator $\Sigma^\theta_{y_1,y_2}$ as: 
\begin{equation}    \Sigma^\theta_{y_1,y_2}( x) \triangleq p_\theta(y_1,y_2\vert  x)-  p_\theta(y_1\vert  x)\cdot p_\theta(y_2\vert  x)
\end{equation}
\end{definition}
As shown in~\cite{johnson2024experts}, this operator is the covariance of the probability distributions in the equivalence class $[ x]$, as stated in the following theorem:
\begin{theorem}\label{thm:covariance operator}
    If the model $p_\theta(y_1,y_2\vert  x)$ is calibrated at predicting pairs, then $\Sigma^\theta_{y_1,y_2}( x)$ is the the covariance of the true distribution $p(y\vert  x)$ in the equivalence class $[ x]$. We write:
    \begin{align}
        \Sigma^\theta_{y_1,y_2}( x) &= \C[p(y_1\vert X),p(y_2\vert X)\vert [ x]] \; , 
    \end{align}
\end{theorem}
This result is important because it demonstrates that a model trained to predict pairs can give reliable information about the distribution of possible distributions, and hence a measure of the epistemic uncertainty in its most general sense. In a classification problem with finite number $K$ of classes,  $\Sigma^\theta( x)$ is a $K\times K$  matrix, so that it can be evaluated explicitly. However, in regression, the covariance operator $\Sigma^\theta( x)$ would be infinite dimensional, which makes it more cumbersome to use in practice. This is the reason why this operator is never used in the present paper, but $\cov_\theta( x)$ is preferred instead. However, the epistemic uncertainty estimate of Thm.~\ref{thm: model covariance} is related to $\Sigma^\theta$ by the following proposition. 
\begin{proposition}\label{eqn:integral operator}
\begin{align}
    \cov_\theta( x)  &= \int_{\mathcal Y\times\mathcal Y}\Sigma^\theta_{y_1,y_2}( x) y_1y_2\d y_1\d y_2  \; ,
 \end{align}
\end{proposition}
%

\subsection{Connection to~\cite{Durasov24a}} 

The idea to estimate the uncertainty in a model by feeding it back its own answers was already proposed by \cite{Durasov24a}. For what concerns regression tasks, the authors propose to score a deterministic network $\hat y=f_\theta(x)$ twice, once with an uninformative constant $\hat y_1 = f_\theta(x\vert y_0)$ and a second concatenating the first answer to the input $\hat y_2=f_\theta(x\vert \hat y_1)$. They then use $u=\sqrt{(\hat y_1-\hat y_2)^2}$ as a measure of the uncertainty.

To start analyzing this methodology, let's make the hypothesis that the underlying phenomenon is itself deterministic, thus having zero aleatoric uncertainty. Then, all probability densities collapse to a Dirac delta, $p(y\vert  x) = \delta(y-f( x))$. Therefore: 
\begin{corollary}\label{thm:deterministic corollary}
Let $p_\theta(y\vert x)$ be a deterministic network, i.e. $p_\theta(y\vert x)=\delta(y-f_\theta(x))$. If the model is first order calibrated on couples $(x,f(x))$, then:
    \begin{equation}\label{eqn:deterministic corollary 1}
    \cov_\theta( x)=\var[f(X)\vert [ x]] \; ,
\end{equation}
where:
\begin{equation}\label{eqn:deterministic equation 2}
   \cov_\theta(x) = f_\theta(x)\cdot f_\theta(x\vert f_\theta)-f_\theta^2(x) \;,
\end{equation}
\end{corollary}

Notice that this result does not require training on triplets because the sampling process can only produce one outcome.

This shows that the intuitive uncertainty metric $u = \vert f_\theta(x \vert f_\theta(x)) - f_\theta(x)\vert $ employed by \cite{Durasov24a} can be formally related to the theoretical covariance metric as:
\begin{equation}
    u = \left\vert\frac{\cov_\theta(x)}{f_\theta(x)}\right\vert\;,
\end{equation}

\section{Validation}\label{sec:experiments}

\subsection{Synthetic Dataset}


At first, a simple synthetic dataset is presented to convey the main ideas of the paper in a controlled setting. 

To this end, the input is sampled uniformly $x \in [-6.0,6.0]$ and the outputs $y$ are drawn from the distribution:
\begin{align}
y  = \mu(x) + & \epsilon(x) \; , \\
\mbox{with} \quad \mu(x) &= x\sin(x) \; ,  \nonumber \\
        \epsilon(x)&\sim \mathcal{N}(0,\sigma^2(x)) \; , \nonumber \\
        \sigma(x)&=1.5\exp{(-x^2/2)}\; ,  \nonumber 
\end{align}
The training set comprises 1000 samples of the form $(x, y_1, y_2)$.

This dataset is used to train a simple MLP with ReLU activation. It has two hidden layers of 64 neurons each, with dropout layers with probability $p=0.05$. Its  output is a mean $\mu_\theta$ and a variance $\sigma^2_\theta$, which is constrained to be positive by passing it through a softplus function. The model is trained for 500 epochs by minimizing the $\beta$-NLL loss~\cite{seitzer2022on} $  \ell_{\beta-\text{NLL}}(y,\mu_\theta(x),\sigma^2_\theta(x)) $. It is taken to be: 
\begin{equation}\label{eqn: beta NLL}
  \text{sg}\left[\sigma_\theta^{2\beta}\right]\left(\frac{1}{2}\log\sigma_\theta^2(x) + \frac{(y-\mu_\theta(x))^2}{2\sigma_\theta^2(x)}\right) \; ,
\end{equation}
where $\text{sg}[\cdot]$ denotes the stop-gradient operation. In other words, the argument is considered to be fixed when computing the gradient:
\begin{equation}
    \nabla_\theta\ell_{\beta-\text{NLL}} = \sigma_\theta^{2\beta}\nabla_\theta\ell_{\text{NLL}}\;,
\end{equation}
with $\beta = 0.5$, which gives a good trade-off between the prediction of $\mu_\theta$ and $\sigma^2_\theta$.



The results are given in Fig.~\ref{fig:results} (left). It can be seen that the correlation between the two answers is only large outside of the training range, whereas the aleatoric prediction presents a maximum around 0, where the data variance is high. In Fig.~\ref{fig:results} (right), the results when training only on tuples $(x_i, y_i)$ are shown: it is interesting to see that in this case the model is not able to distinguish aleatoric and epistemic uncertainty, as predicted in Thm.~\ref{thm:zigzag final theorem}, so that the covariance presents the exact same maximum around 0 as the aleatoric variance. This highlights the importance of the present theoretical results, which demand to change the training procedure in order to obtain a meaningful frequentist measure of the epistemic uncertainty. As an important side note, it must be stressed that all the theorems are valid only if the model is well calibrated enough to begin with, which means that there are no guarantees on the performance of the current methodology for out-of-distribution samples. Nevertheless, it can e argued that high values of correlation between the two responses of the model will be a sign of epistemic uncertainty (when triples are used during training) because, under no circumstances, it can derive from the true underlying distribution if $Y_1\vert X$ and $Y_2\vert X$ are iid.
\begin{figure*}
    \centering
    \includegraphics[width=0.7\textwidth]{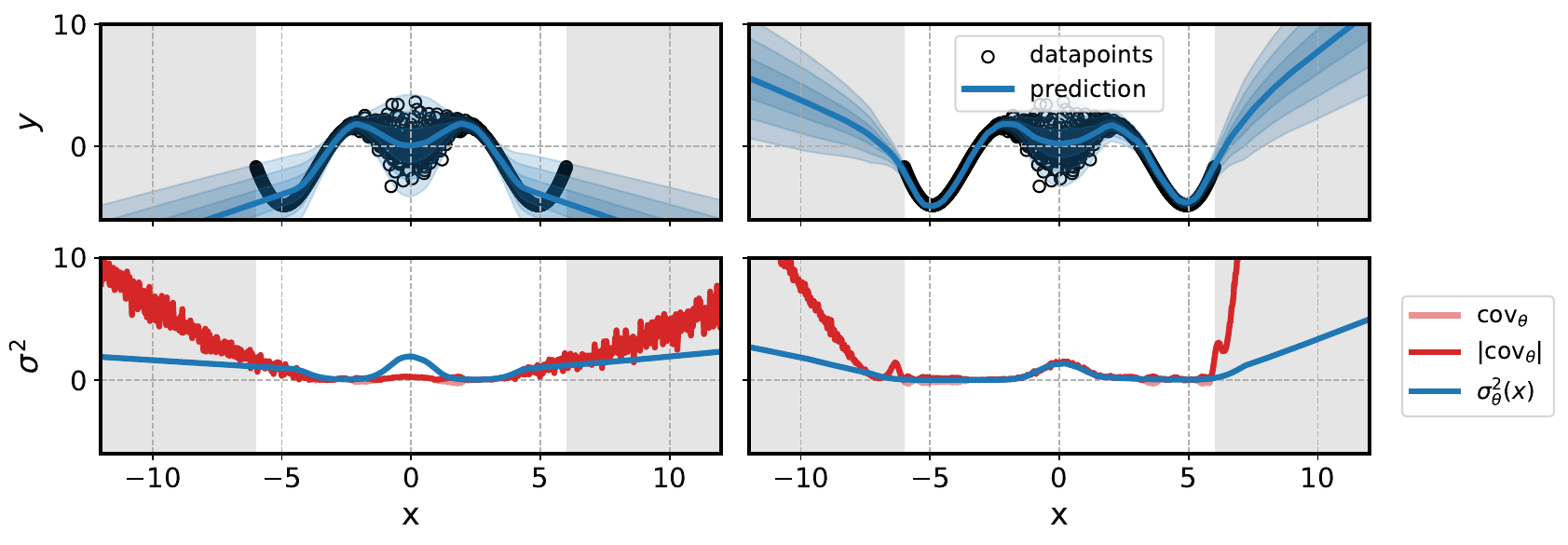}
    \caption{\small \textbf{Toy model: comparing training on couples and triplets.} Results of the simplified experiment. On the left, the model is trained on triplets $(x, y_{1}, y_{2})$, whereas on the right we only used tuples.}
    \label{fig:results}
\end{figure*}

Motivated by recent work on the disentangling of epistemic and aleatoric uncertainty \cite{mucsanyi2024benchmarking}, the methodology was tested with different levels of corruption $\sigma(x) = \gamma\exp(-x^2/2)$, with $\gamma=1, 1.5, 3$. The results are shown in Fig.~\ref{fig:different severity}. Reassuringly, the level of noise in the data does not affect the epistemic uncertainty. It is possible to argue that the aleatoric and epistemic are still very correlated far from the mode of $p(X)$, however the model variance $\sigma_\theta^2$ ceases to be very informative for out-of-distribution samples, and should mostly be discarded.
\begin{figure*}
    \centering
    \includegraphics[width=0.7\linewidth]{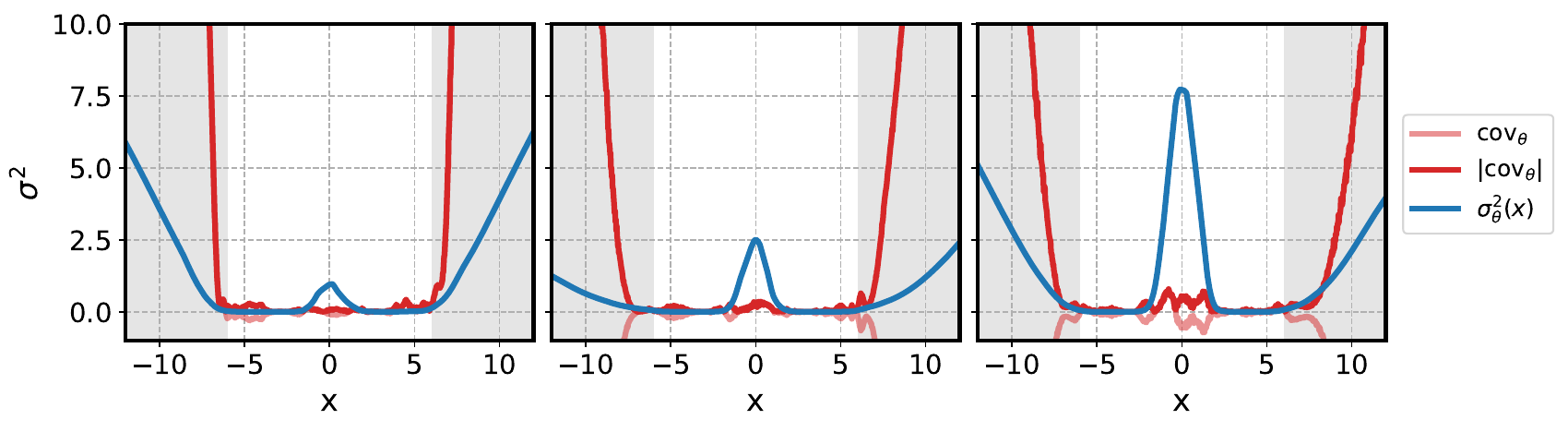}
    \caption{\small \textbf{Toy model: results for increasing data corruption.} Comparison of epistemic and aleatoric uncertainty for different levels of data corruption. The estimation of the epistemic uncertainty $\cov_\theta$ is unaffected by the level of the aleatoric component $\sigma$.}
    \label{fig:different severity}
\end{figure*}

\subsection{Aerodynamics of an airfoil}
Next, this methodology is applied to a real world dataset issued from lift and drag measurements in a low-speed wind tunnel. A model of a NACA0012 airfoil is placed in the test section, and can be rotated with respect to the incoming flow. The angle between the chord of the profile and air speed vector is called angle of attack. Also, the air-speed can be controlled, which changes significantly the types of phenomena that can be observed in the experiment.

For every flow condition, i.e. a couple $(\alpha,U_\infty)$, where $\alpha$ is the angle of attack and $U_\infty$ the inflow velocity, 10 seconds of data were collected at 1 kHz of acquisition frequency. This is standard practice to ensure statistical convergence, since each instantaneous measurements vary from one another because of various factors, like the precision of the force balances, turbulence or external disturbances. For each flow condition, 250 points were selected randomly from the time signal of the measurement to represent the first set of outcomes, and other 250 for the second set. In the training set the angle of attack was set $\alpha\in[-11^\circ,11^\circ]$ in increments of 0.5, and four flow speeds, namely $U_\infty = 7.3, 9.1, 11.9$ and $14.1$ ms$^{-1}$, were also used. For testing, the same range of angles of attack were employed, but at $U_\infty=4.7$ ms$^{-1}$. The model architecture is similar to the one used on the previous dataset. As shown in Tab.~\ref{tab:results airfoil}, the covariance is consistently bigger in the test dataset, where $U_\infty$ is smaller than any velocity seen during training, than in the train dataset. Fig.~\ref{fig:results lift 14.2} shows the predictions for one in-dataset distribution, at $U_\infty=14.2$~ms$\inv$: the epistemic uncertainty $\cov_\theta$ remains small on the entire range of $\alpha$ seen during training, and starts growing outside of this range.

\begin{table}[]
    \caption{Overview of the results of the airfoil experiment. Results are averaged over 5 runs.}
    \small
    \vskip 0.15in
    \centering
    \begin{tabular}{cccc}
     \toprule                 & Train       & Test        & Difference               \\
\midrule\multirow{2}{*}{R2}  & 0.99      & 0.83      & \multirow{2}{*}{-15.8\%} \\
                     & $\pm$0.00 & $\pm$0.04 &                          \\
\multirow{2}{*}{$\E[\vert\cov_\theta\vert]\; (\times 10^{-2})$} & 0.09        & 1.17        & \multirow{2}{*}{+1140\%} \\
                     & $\pm$0.02   & $\pm$0.98   &      \\
                     \bottomrule
    \end{tabular}
    \label{tab:results airfoil}
\end{table}

\begin{figure}
    \centering
    \includegraphics[width=0.9\linewidth]{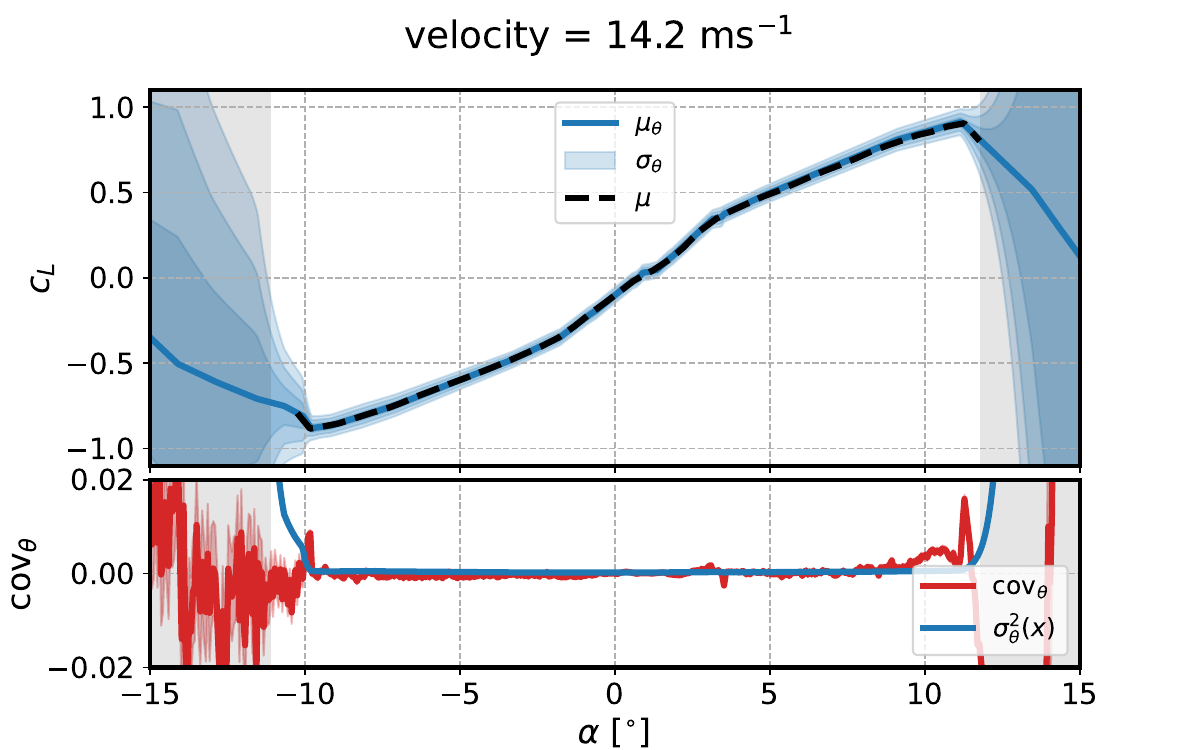}
    \caption{\small \textbf{Airfoil aerodynamics: result of in-dataset sample.} The figure shows the prediction of the model on one of the in-dataset velocities. Both the mean and the variance are well captured. Shaded areas represent the $\sigma$, $2\sigma$ and $3\sigma$ confidence intervals given by the total uncertainty $\sigma_\theta$. The epistemic covariance remains small in the range of $\alpha$ seen during training, and grows rapidly outside of it.}
    \label{fig:results lift 14.2}
\end{figure}
%

\subsection{Drone noise}
As a last dataset, the drone noise measurements performed by \cite{gojon2021experimental}, and available on the Dataverse \cite{king2007introduction}, are presented. For the purposes of the present work, only the propellers ISAE propellers with 2 to 5 blades are retained. The experimental conditions are then given by three parameters: the number of blades $n$, the rotational speed $\Omega$ and the angle of the receiving microphone with respect to the rotor plane $\vartheta$. In total, this results in 780 configurations. The predicted quantities are the amplitudes of five peaks on the noise spectrum, corresponding to the first harmonics of the blade passing frequency (BPF) $m\cdot \mathrm{BPF} = m\cdot n\Omega$, with $m=1,\dots,5$. These are quantities of interest because they make up the most disturbing components of the sound for the human ears. The raw microphone data is made up of long time recordings of the far-field acoustic pressure fluctuations. To achieve a dataset made of triplets, all time series are split in two parts which, supposing the process to be ergodic, can be considered as two realizations of the experiment. Both are then processed with a Fast Fourier Transform to extract the amplitudes of the first emerging peaks. The train-test split is performed as follows: in a first run, the test dataset is composed of all the data concerning the propeller with $n=3$ blades, whereas in a second one the testing is performed on the two-blades rotor. The results in terms of correlation coefficient and epistemic uncertainty are resumed in Tab.~\ref{tab:results drone noise}. In particular, note that the epistemic uncertainty is consistently bigger in the test set, but the gap is smaller if the held-out data is found ``in-between'' other data-points, where the model is supposed to generalize better. Fig.~\ref{fig:dual directivity 3 blades} shows a test sample for the three-bladed rotor at $\Omega=5000$ rpm, at three times the BPF, to illustrate this concept.
\begin{table}[!hbt]
    \caption{Accuracy and predicted epistemic uncertainty for the train-test split using the two and the three-bladed rotors. All results are given in standardized units and averaged over 5 runs.}
    \vskip 0.15in
    \centering
    \begin{small}
    \begin{tabular}{cccc}
        \toprule &  Train & Test & Difference \\
        \midrule \multicolumn{4}{c}{\textit{2-blades}}\\
        \midrule \multirow{2}{*}{R2 score} & 0.97 & 0.77 &  \multirow{2}{*}{-20\%}\\
        & $\pm$0.00& $\pm$0.02& \\
        \multirow{2}{*}{$\E[\vert\cov_\theta\vert]\; (\times 10^{-2})$}& 2.358 & 5.634 & \multirow{2}{*}{+139\%}\\
        &$\pm$0.11 & $\pm$0.52& \\
        \midrule \multicolumn{4}{c}{\textit{3-blades}}\\
        \midrule \multirow{2}{*}{R2 score} & 0.97 & 0.89 &  \multirow{2}{*}{-8.9\%}\\
        & $\pm$0.00 & $\pm$0.01 &\\
        \multirow{2}{*}{$\E[\vert\cov_\theta\vert]\; (\times 10^{-2})$} & 2.644&3.956 &\multirow{2}{*}{+49\%}\\
        & $\pm$0.11 & $\pm$0.49&\\
        \bottomrule
    \end{tabular}
    \end{small}
    \label{tab:results drone noise}
\end{table}

\begin{figure}
    \centering
    \includegraphics[width=0.45\textwidth]{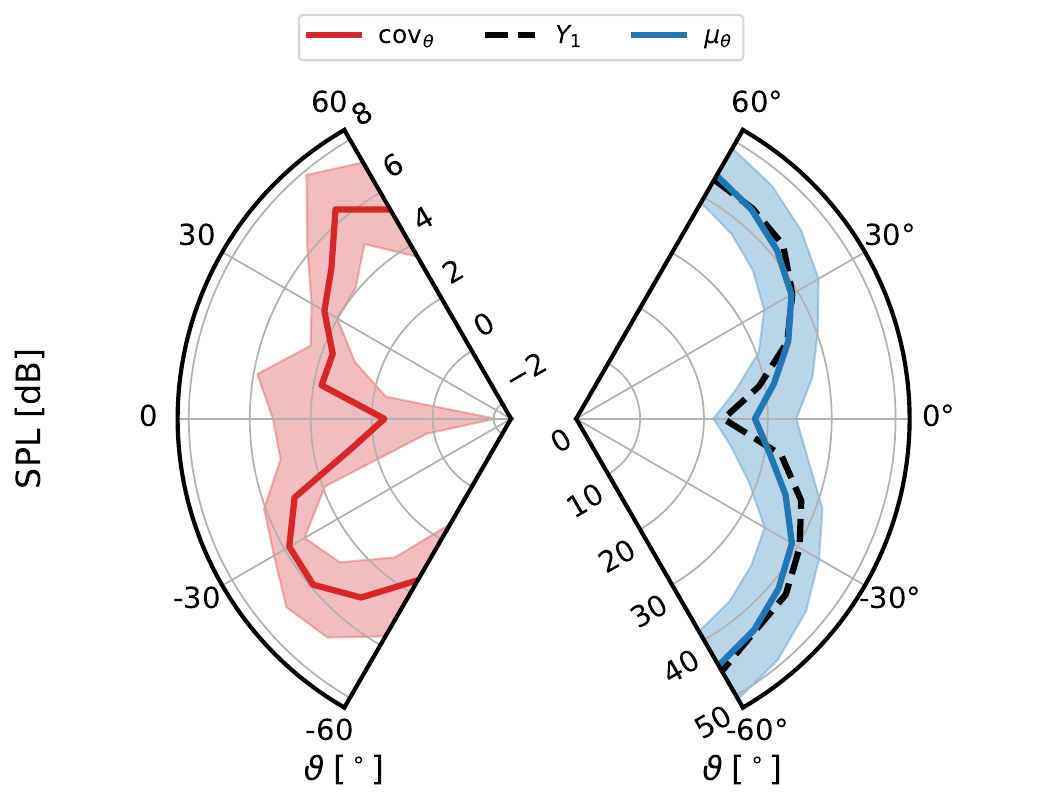}
    \caption{\small \textbf{Drone noise: results for near-out-of-dataset sample.} Test sample for the three-bladed rotor at $\Omega=5000$ rpm, at three times the BPF. The shaded areas represent the $2\sigma$ confidence interval; given by the variance of the Monte Carlo integration on the left and by the total uncertainty $\sigma_\theta^2$ on the right.}
    \label{fig:dual directivity 3 blades}
\end{figure}

\section{Conclusion}\label{sec:conclusion}
The present work establishes a mathematically rigorous approach to estimate the epistemic uncertainty of a model in a frequentist manner, for regression tasks. In particular, a perfectly calibrated model has been shown to mix aleatoric and epistemic uncertainty when predicting its variance (Thm. \ref{tmh: model moments}). By training on a dataset composed of triplets $(x, y_1,y_2)$, where $y_1$ and $y_2$ are iid, it is possible to estimate the epistemic part of the variance by using the model covariance (Thm. \ref{thm: model covariance}). This results extends previous work done in the context of classification by \cite{johnson2024experts}. Finally, a practical way to compute the model covariance by feeding the model its own predictions has been presented in Eq.~(\ref{eqn:cov approximation}), which requires minimal changes to the model architecture and training procedure.

Looking at the problem of epistemic uncertainty prediction under a frequentist lens can help to diversify the landscape of UQ methodologies, which at the moment is dominated by Bayesian approaches. While these are undeniably successful, there seems to always be a trade-off between accuracy in the prediction of the weight posterior $p(\theta\vert \mathcal D)$ and the computational cost. The present method, on the other hand, shifts the burden from the model, which is modified only slightly, to the dataset, where we require multiple outcomes to be collected for every input. While this could be an insurmountable obstacle for some practitioners, many experimental datasets are already built this way, as we showed for the aerodynamic loading and the drone noise measurements.

The requirement for the model to be calibrated in the first place can seem harsh. However the present methodology is able to detect when a model is \emph{not} calibrated, by violating the theorems proved in the case of perfect calibration, which would be difficult to do otherwise. This behavior results in high levels of $\vert \cov_\theta\vert$ when extrapolating far from the dataset. Explaining the behavior of our feedback procedure for non calibrated models could be a fruitful direction for future research.

\bibliography{ICML_2025/bibliography,bib/string,bib/learning}
\bibliographystyle{ICML_2025/icml2025}

\newpage
\appendix
\onecolumn

\section{Proofs of the theorems}\label{sec:proofs}

\begin{proof}[Proof of theorem \ref{tmh: model moments}]
    \begin{equation}\label{proof: mean first order}
        \begin{aligned}
            \E_\theta[Y\vert x] &= \int_\mathcal Y yp_\theta(y\vert x)\d y &\text{definition of $\E$}\\
            &=\int_\mathcal Y y\left(\int_\mathcal X p(y\vert x')p(x'\vert [x])\d x'\right)\d y &\text{Eq.~(\ref{eqn:first order calibration})}\\
            &=\int_\mathcal{X} p(x'\vert [x])\left(\int_\mathcal{Y}yp(y\vert x')\d y\right)\d x' &\text{Fubini's thm.}\\
            &=\E[\E[Y\vert X]\vert [x]]
        \end{aligned}
    \end{equation}
    The definition of the model variance is:
    \begin{equation}\label{eqn:var proof step 1}
        \begin{aligned}
                \var_\theta[Y\vert x] &= \E_\theta[Y^2\vert x]-\E^2_\theta[Y\vert x] &\text{definition of $\var$}\\
                &=\E[\E[Y^2\vert X]\vert [x]]-\E^2[\E[Y\vert X]\vert [x]] &\text{as for Eq.~(\ref{proof: mean first order})}\\
        \end{aligned}
    \end{equation}
    Furthermore, by the definition of variance and the linearity of expectation:
    \begin{equation}\label{eqn:var proof step 2}
            \E[\var[Y\vert X]\vert [x]] = \E[\E[Y^2\vert X]\vert [x]]-\E[\E^2[Y\vert X]\vert [x]]
    \end{equation}
    Substituting Eq.~(\ref{eqn:var proof step 2}) into Eq.~(\ref{eqn:var proof step 1}) yields:
    \begin{equation}
        \begin{aligned}
            \var_\theta[Y\vert x] &= \E[\var[Y\vert X]\vert [x]]+\E[\E^2[Y\vert X]\vert [x]]-\E^2[\E[Y\vert X]\vert [x]]\\
            &= \E[\var[Y\vert X]\vert [x]] + \var[\E[Y\vert X]\vert [x]]&\text{definition of $\var$}
        \end{aligned}
    \end{equation}
\end{proof}


\begin{proof}[Proof of theorem \ref{thm:marginal calibration}]
     \begin{equation}
        \begin{aligned}
            p_\theta(y_1\vert x)&\triangleq\int_\mathcal Y p_\theta(y_1,y_2\vert x)\d y_2&\text{definition of marginal}\\
            &=\int_\mathcal Y \left(\int_\mathcal X p(y_1\vert x')\cdot p(y_2\vert x')p(x'\vert [x])\d x'\right)\d y_2&\text{definition Eq.~(\ref{eqn: first order calibration pairs})}\\
            &=\int_\mathcal X p(y_1\vert x')\left(\int_\mathcal Y p(y_2\vert x')\d y_2\right)p(x'\vert [x])\d x'&\text{Fubini}\\
            &= \int_\mathcal X p(y_1\vert x')p(x'\vert [x])\d x'&\text{$\int_\mathcal Yp(y)\d y=1$}\\
            &=\E[p(y\vert X)\vert [x]]
        \end{aligned}
    \end{equation}
\end{proof}

\begin{proof}[Proof of Theorem \ref{thm:covariance operator}]
    The proof follows immediately from theorem \ref{thm:marginal calibration}:
    \begin{equation}
        \begin{aligned}
            \Sigma^\theta_{y,y'}(x)&\triangleq p_\theta(y,y'\vert x)-p_\theta(y\vert x)p_\theta(y'\vert x)\\
            &=\E[p(y,y'\vert X)\vert [x]]-\E[p(y\vert X)\vert [x]]\cdot\E[p(y'\vert X)\vert [x]]&\text{def. \ref{def:calibration} and thm. \ref{thm:marginal calibration}}\\
            &=\E[p(y\vert X)p(y'\vert X)\vert [x]]-\E[p(y\vert X)\vert [x]]\cdot\E[p(y'\vert X)\vert [x]]&y\vert X\text{ and }y'\vert X\text{are iid} \\
            &=\C[p(y\vert X),p(y'\vert X)\vert [x]]&\text{def. of }\C
        \end{aligned}
    \end{equation}
\end{proof}

\begin{proof}[Proof of Theorem \ref{thm: model covariance}]
    \begin{equation}
        \mathbb C_\theta[Y_1,Y_2\vert x] \triangleq \E_\theta[Y_1Y_2\vert x]-\E_\theta[Y_1\vert x]\cdot\E_\theta[Y_2\vert x]\quad\text{definition of $\mathbb C$}\\
\end{equation}
begin by the first term:
\begin{equation}\label{eqn:mean of Y1 x Y2}
    \begin{aligned}
        \E_\theta[Y_1Y_2\vert x] &= \int_{\mathcal Y\times\mathcal Y}y_1y_2p_\theta(y_1,y_2\vert x)\d y_1 \d y_2&\text{definition of $\E$}\\
        &= \int_{\mathcal Y\times\mathcal Y}y_1y_2\left( \int_\mathcal X p(y_1\vert x')p(y_2\vert x')p(x'\vert [x])\d x'\right)\d y_1 \d y_2&\text{Eq.~(\ref{eqn: first order calibration pairs})}\\
        &=\int_\mathcal X p(x'\vert [x])\left(\int_\mathcal Yy_1p(y_1\vert x')\d y_1\int_\mathcal Yy_2p(y_2\vert x')\d y_2\right)\d x'&\text{Fubini}\\
        &=\E[\E^2[Y_1\vert X]\vert [x]]&Y_1\vert X\text{ and }Y_2\vert X\text{ are iid}
    \end{aligned}
\end{equation}
    Plugging this result back:
    \begin{equation}
        \begin{aligned}
             \mathbb C_\theta[Y_1,Y_2\vert x] &= \E[\E^2[Y_1\vert X]\vert [x]]-\E_\theta[Y_1]\E_\theta[Y_2\vert x]\\
             &=\E[\E^2[Y_1\vert X]\vert [x]]-\E^2[\E[Y_1\vert X]\vert [x]]&\text{by Thm.~\ref{thm:marginal calibration}}\\
             &=\var[\E[Y_1\vert X]\vert[x]]&\text{definition of $\mathbb \var$}
        \end{aligned}
    \end{equation}
\end{proof}

\begin{proof}[Proof of corollary \ref{thm:zigzag final theorem}]
    Eq.~(\ref{eqn:deterministic corollary 1}) follows immediately from Thm.\ref{thm: model covariance} by noting that, being $p(Y\vert X) = \delta(Y-f(X))$, its average is $\E[Y\vert X] = f(X)$ and the variance is exactly zero, $\var[Y\vert X] = 0$. Eq.~(\ref{eqn:deterministic equation 2}) can be derived from Eq.~(\ref{eqn:cov approximation}) noting that one can only sample one value from a Dirac distribution, namely $f_\theta(x)$.
\end{proof} 


\begin{proof}[Proof of Theorem \ref{thm:zigzag final theorem}]
    The proof follows from the one one used for Thm.~\ref{tmh: model moments} by noticing that:
    \begin{equation}
    \begin{aligned}
         \E_\theta[Y_1Y_2] &= \int_{\mathcal Y\times\mathcal Y}y_1y_2p_\theta(y_1,y_2\vert x)\d y_1\d y_2 \\
        &= \int_{\mathcal Y\times\mathcal Y}y_1y_2\left(\int_\mathcal{X}p(y_1,y_2\vert x')p(x'\vert [x])\d x'\right)\d y_1\d y_2&\text{def.~(\ref{def:calibration})}\\
        &= \int_{\mathcal Y\times\mathcal Y}y_1y_2\left(\int_\mathcal{X}p(y_1\vert x')\delta(y_2-y_1)p(x'\vert [x])\d x'\right)\d y_1\d y_2&\text{train on couples}\\
        &=\int_\mathcal{Y}y_1^2\left(\int_{\mathcal{X}}p(y_1\vert x')p(x'\vert [x])\d x'\right)\d y_1&\text{def. of }\delta\\
        &=\int_{\mathcal{Y}}y_1^2p_\theta(y_1\vert x)\d y_1&\text{def.~(\ref{def:calibration})}\\
        &=\E_\theta[Y_1^2]
    \end{aligned}
    \end{equation}
    and that, following the same logic:
    \begin{equation}
        \E_\theta[Y_1]\cdot \E_\theta[Y_2] = \E_\theta^2[Y_1]
    \end{equation}
    
\end{proof}

\begin{proof}[Proof of Eq.~(\ref{eqn:cov approximation})]
\begin{equation*}
    \begin{aligned}
        \cov_\theta( x) & = \E_\theta[Y_1\cdot Y_2\vert  x]-\E_\theta[Y_1\vert  x]\cdot \E_\theta[Y_2\vert  x  &\text{def. of $\cov_\theta$}\\
        &= \int_{\mathcal Y\times\mathcal Y}y_1y_2p_\theta(y_1,y_2| x) \d y_1\d y_2 - \mu_\theta^2( x)  &\text{def. of $\E$}\\
        &= \int_\mathcal Y \left[\int_\mathcal Y \!  y_2p_\theta(y_2|y_1, x)\d y_2\right]  y_1p_\theta(y_1\vert  x)\d y_1  -  \mu_\theta^2( x) &\text{Eq.~(\ref{eqn:decomposition joint zigzag})} \\
        &= \int_\mathcal Y \E_\theta[Y_2\vert y_1, x]y_1p_\theta(y_1\vert  x)\d y_1- \mu_\theta^2( x)  &\text{def. of $\E$} \\
        & \simeq \frac{1}{M}\sum_{m=1}^M y_m\mu_\theta( x\vert y_m)- \mu_\theta^2( x),  \; y_m\sim p_\theta(y\vert  x) &\text{Monte Carlo}
    \end{aligned}
\end{equation*}            
\end{proof}


\begin{proof}[Proof of Eq.~(\ref{eqn:integral operator})]
    \begin{equation}
    \begin{aligned}
        \int_{\mathcal Y\times\mathcal Y}\Sigma^\theta_{y_1,y_2}(X)y_1y_2\d y_1\d y_2 &=\int_{\mathcal Y\times\mathcal Y}\C[p(y_1\vert X),p(y_2\vert X)\vert [x]]y_1y_2\d y_1\d y_2 \\
        &=\int_{\mathcal Y\times\mathcal Y}\E[p(y_1\vert X)p(y_2\vert X)\vert [x]]y_1y_2\d y_1\d y_2+&\text{def. of $\C$}\\&\quad-\int_{\mathcal Y\times\mathcal Y}\E[p(y_1\vert X)\vert[x]]\E[p(y_2\vert X\vert [x]])y_1y_2\d y_1\d y_2\\
    &=\E\left[\int_{\mathcal Y\times\mathcal Y}p(y_1\vert X)p(y_2\vert X) y_1y_2\d y_1\d y_2\bigg\vert[x]\right]+&\text{Fubini}\\&\quad-\E\left[\int_{\mathcal Y}p(y_1\vert X)y_1\d y_1\bigg\vert[x]\right]\E\left[\int_{\mathcal Y}p(y_2\vert X)y_2\d y_1\bigg\vert[x]\right]\\
    &=\E[\E[Y_1\vert X]\E[Y_2\vert X]\vert[x]]-\E[\E[Y_1\vert X]\vert[x]]\E[\E[Y_2\vert X]\vert[x]]&\text{def. of $\E$}\\
    &=\E[\E^2[Y_1\vert X]\vert[x]]-\E^2[\E[Y_1\vert X]\vert[x]]&Y_1\vert X\text{and $Y_2\vert X$ are iid}\\
    &=\var[\E[Y_1\vert X]\vert[x]]&\text{def. of $\var$}\\&=\C_\theta[Y_1,Y_2\vert x] = \cov_\theta(x)
    \end{aligned}
\end{equation}
\end{proof}
Since Thm.~\ref{thm: model covariance} relates the model covariance to the grouping loss, i.e. the variance of the averages in the equivalence class $[x]$, Chebyshev inequality can be used to get an estimation of the error on the prediction of the mean:
\begin{corollary}\label{thm:chebychev}
    Let $p_\theta(y_1,y_2\vert x)$ be a model calibrated at predicting pairs, with marginal expectation $E_\theta[Y\vert x]=\mu_\theta( x)$ and covariance $\C_\theta[Y_1,Y_2\vert x]=\cov_\theta( x)$. Also, let $\mu( x)=\E[Y\vert x]$ be the mean of the data distribution given $X= x$ and $\beta>0$ any positive real number. Then:
    \begin{equation}
        \E[(\mu(X)-\mu_\theta( x))^2\vert[ x]]=\cov_\theta( x)\;.
    \end{equation}
    Furthermore:
    \begin{equation}
        \mathbb P\left[\vert \mu_\theta( x)-\mu(X)\vert\geq \left.\sqrt{\frac{\cov_\theta( x)}{\beta}}\right\vert X\in[ x]\right]\leq \beta\;.
    \end{equation}
\end{corollary}

\begin{proof}[Proof of corollary~\ref{thm:chebychev}]
    Both parts of the theorems are simple consequences of Thm.~\ref{thm: model covariance}. The first part follows from the definition of variance:
    \begin{equation}
        \begin{aligned}
            \cov_\theta( x) &= \var[\mu(X)\vert[ x]]&\text{Thm.~\ref{thm: model covariance}}\\
            &=\E[(\mu(X)-\E[\mu(X)\vert[ x]])^2\vert[ x]]&\text{def. of $\var$}\\
            &=\E[(\mu(X)-\mu_\theta( x))^2\vert[ x]]&\text{Thm.~\ref{tmh: model moments}}
        \end{aligned}
    \end{equation}
    For the second part, remembering that for any random variable $Z$ with finite variance (and expectation) $\var[Z]$
($\E[Z]$), the Chebyshev inequality yields:
    \begin{equation}
        \prob\left[\vert Z-\E[Z]\vert \geq \sqrt{\frac{\var[Z]}{\beta}}\right]\leq\beta
    \end{equation}
    Let $Z=\E[Y\vert X]=\mu(X)$. Conditioning on $X\in[ x]$ gives:
    \begin{equation}
        \begin{aligned}
            \beta&\geq\prob\left[\vert \mu(X)-\E[\mu(X)\vert[ x]]\vert\geq\left.\sqrt{\frac{\var[\mu(X)\vert[ x]]}{\beta}}\right\vert X\in[ x] \right]\\
            &=\prob\left[\vert \mu(X)-\mu_\theta( x)\vert\geq\left.\sqrt{\frac{\cov_\theta( x)}{\beta}}\right\vert X\in[ x] \right]&\text{Thm.~\ref{tmh: model moments} and Thm.~\ref{thm: model covariance}}
        \end{aligned}
    \end{equation}
\end{proof}

\section{Variance calibration and distribution calibration}\label{sec:variance calibration}

The notion of distribution calibration in regression can be, at best, difficult to check in practice. A more common notion of calibration is quantile calibration which, roughly speaking, demands that within an $x$-percent confidence interval around the predicted mean one must find $x$-percent of the true data points. While this notion of calibration is intuitive and easy to check in practice, it has the disadvantage that it averages over the entire dataset. It is possible to construct models that predict very poorly the true distribution, and yet are quantile calibrated. In \cite{levi2022evaluating}, the authors propose the notion of variance calibration as:
\begin{equation}\label{eqn:variance calibration}
    \sigma^2_\theta(x) = \E[(\mu_\theta(X)-Y)^2\vert [x]_\sigma]
\end{equation}
where the equivalence class is $[x]_\sigma = \{x'\in\mathcal X\;\vert\;\sigma_\theta^2(x')=\sigma_\theta^2(x)\}$. While this definition is stronger than quantile calibration, it has the disadvantage that it allows the model to explain away its errors using its variance. Even a model that only predicts $\mu_\theta\equiv 0$ can be perfectly calibrated in the sense of Eq.~(\ref{eqn:variance calibration}), by setting $\sigma_\theta(x) = \E[Y^2\vert [x]_\sigma]$. It seems natural, then, to include the prediction of the mean in the definition of calibration, which gives the following definition:
\begin{definition}
    A model $p_\theta(Y\vert X)$ is said to be strongly variance calibrated if its mean $\mu_\theta(x)$ and variance $\sigma_\theta^2(x)$ obey:
    \begin{align}
            \mu_\theta(x) &= \E[\E[Y\vert X]\vert [x]_p]\label{eqn:mu calibrated}\\
            \sigma^2_\theta(x) &= \E[\E[(\mu_\theta(X)-Y)^2\vert X]\vert [x]_p]
    \end{align}
    where the equivalence class aggregates all points with the same mean and variance, i.e. $[x]_p = \{x'\in\mathcal X\;\vert\;\sigma_\theta^2(x')=\sigma_\theta^2(x)\text{ and }\mu_\theta(x')=\mu_\theta(x)\}$.
\end{definition}
This definition of calibration is harder to obtain than the one used in \cite{levi2022evaluating}, because to be evaluated it requires binning over two variables, $\mu_\theta$ and $\sigma_\theta^2$. On the other hand, this definition of calibration allows to recover Thm.~\ref{tmh: model moments}:
\begin{theorem}
    Let $p_\theta(Y\vert X)$ be strongly variance calibrated. Then it holds that:
    \begin{equation}\label{eqn:variance variance calibration}
        \sigma_\theta^2(x) = \E[\var[Y\vert X]\vert [x]_p]+\var[\E[Y\vert X]\vert [x]_p]
    \end{equation}
\end{theorem}
\begin{proof}
    \begin{equation}
\begin{aligned}
        \sigma_\theta^2(x)& = \E[\E[(\mu_\theta(X)-Y)^2\vert X]\vert [x]_p] \\
        &=\E[\E[\mu^2_\theta(X)\vert X]\vert [x]_p] + \E[\E[Y^2\vert X]\vert [x]_p]-2\E[\E[\mu_\theta(X)Y\vert X]\vert [x]_p]&\text{linearity of $\E$}\\
        &=\mu_\theta^2(x)+\E[\E[Y^2\vert X]\vert [x]_p]-2\mu_\theta(x)\E[\E[Y\vert X]\vert [x]_p]&\text{$\mu_\theta$ is a function of $[x]_p$}\\
        &=\E[\E[Y^2\vert X]\vert [x]_p]-\mu^2_\theta(x)&\text{Eq.~(\ref{eqn:mu calibrated})}\\
        &=\E[\E[Y^2\vert X]\vert [x]_p]-\E[\E^2[Y\vert X]\vert [x]_p]+\E[\E^2[Y\vert X]\vert [x]_p]-\mu^2_\theta(x)&\text{add and subtract}\\
        &=\E[\var[Y\vert X]\vert [x]_p]+\var[\E[Y\vert X]\vert [x]_p]&\text{def. of $\var$}
\end{aligned}
\end{equation}
\end{proof}
It is possible to extend the notion of variance calibration to models trained to predict pairs as:
\begin{definition}\label{def:covariance calibration}
   Let $p_\theta(y_1,y_2\vert x)$ be a model whose means is $\vec{\mu_\theta}(x)=[\mu_{\theta,1}(x),\mu_{\theta,2}(x)]^\top$ and whose covariance matrix is given by:
    \begin{equation}
        \Sigma_\theta(x) = \begin{bmatrix}
            \sigma_{\theta,1}^2(x) & \cov_\theta(x)\\ \cov_\theta(x) & \sigma_{\theta,2}^2(x) 
        \end{bmatrix}
    \end{equation}
    Let $\vec Z$ be a random vector defined as $\vec Z = \E[[\mu_{\theta,1}-Y_1,\; \mu_{\theta,2}-Y_2]^\top\vert X]$. The model is said to be covariance calibrated if:
    \begin{align}
        \vec{\mu_\theta}(x) &= \E[\E[[Y_1,Y_2]^\top\vert X]\vert [x]_c]\\
        \Sigma_\theta(x) &= \E[ZZ^\top\vert [x]_c] = \E\left[\E\left[\left.\begin{bmatrix}
            (\mu_{\theta,1}(X)-Y_1)^2 & (\mu_{\theta,1}(X)-Y_1)(\mu_{\theta,2}(X)-Y_2)\\
            (\mu_{\theta,1}(X)-Y_1)(\mu_{\theta,2}(X)-Y_2) & (\mu_{\theta,2}(X)-Y_2)^2
        \end{bmatrix}\right \vert X\right][x]_c\right]
    \end{align}
    where the equivalence class $[x]_c$ includes all inputs $x$ that result in the same mean vector and covariance matrix.
\end{definition}
The diagonal terms have already been analyzed in Eq.~(\ref{eqn:variance variance calibration}). Notice that, because of the calibration condition and the fact that $Y_1\vert X$ and $Y_2\vert X$ are iid, $\mu_{\theta,1}=\mu_{\theta,2}=\mu_\theta$ and $\sigma^2_{\theta,1}=\sigma^2_{\theta,2}=\sigma^2_{\theta}$: 
\begin{lemma}\label{thm:lemma covariance calibration}
    Let $p(Y_1,Y_2\vert X)$ be a model calibrated in the sense of Def.~\ref{def:covariance calibration}. If the data distribution $p(Y_1,Y_2\vert X)$ can be decomposed as $p(Y_1\vert X)p(Y_2\vert X)$, then:
    \begin{align}
        \mu_{\theta,1}(x) = \mu_{\theta,2}(x) &= \E[\E[Y_1\vert X]\vert [x]_c]\label{eqn:calibration mean 2}\\
        \sigma^2_{\theta,1}(x) = \sigma^2_{\theta,1}(x) &= \E[\var[Y_1\vert X]\vert [x]_c] + \var[\E[Y_1\vert X]\vert [x]_c]\label{eqn:calibration var 2}
    \end{align}
\end{lemma}
\begin{proof}
    Eq.~(\ref{eqn:calibration mean 2}) follows immediately from the fact that $Y_1\vert X$ and $Y_2\vert X$ are iid. Similarly, Eq.~(\ref{eqn:calibration var 2}) follows from Thm.~\ref{tmh: model moments}.
\end{proof}
It makes sense, then, to talk about $\mu_\theta(x)$ and $\sigma_\theta^2(x)$ without specifying the index. 
It turns out that the off-diagonal terms in the covariance matrix $\Sigma_\theta(x)$, i.e. $\cov_\theta(x)$, behave according to Thm.~\ref{thm: model covariance}:
\begin{theorem}
    Under the same hypothesis of Lemma~\ref{thm:lemma covariance calibration}, the off-diagonal terms of $\Sigma_\theta(x)$ read:
    \begin{equation}
        \cov_\theta(x) = \var[\E[Y\vert X]\vert [x]_c]
    \end{equation}
\end{theorem}
\begin{proof}
    \begin{equation}
    \begin{aligned}
        \cov_\theta(x) &= \E[\E[(Y_1-\mu_{\theta}(X))(Y_2-\mu_{\theta}(X))\vert X]\vert [x]_c]&\text{Def.~\ref{def:covariance calibration}}\\
        &=\E[\E[Y_1Y_2 - Y_1\mu_{\theta}(X)-Y_2\mu_{\theta}(X)+\mu_{\theta}^2(X)\vert X]\vert [x]_c]\\
        &=\E[\E[Y_1Y_2\vert X]\vert [x]_c]- \mu_{\theta}(x)\E[\E[Y_1\vert X]\vert[x]_c]-\mu_{\theta}(x)\E[\E[Y_2\vert X]\vert[x]_c]+ \mu_{\theta}^2(x)&\mu_\theta \text{ is function of }[x]_c\\
        &=\E[\E^2[Y\vert X]\vert[x]_c]-\mu_\theta^2(x)&\text{Eq.~(\ref{eqn:mean of Y1 x Y2}) and Lemma~\ref{thm:lemma covariance calibration}}\\
        &=\E[\E^2[Y\vert X]\vert[x]_c] - \E^2[\E[Y\vert X]\vert [x]_c]&\text{Lemma~\ref{thm:lemma covariance calibration}}\\
        &=\var[\E[Y\vert X]\vert [x]_c]&\text{Def. of }\var
    \end{aligned}
\end{equation}
\end{proof}
This shows that the main theorems of the paper, namely Thm.~\ref{tmh: model moments} and Thm.~\ref{thm: model covariance}, still hold when using a weaker notion of calibration, namely variance (or covariance) calibration

\section{Outline of the necessary model modifications}

The model presented in this paper must be able to predict two outputs for every input, which must be potentially correlated. This rules out the possibility to just train two different models. In principle, it is possible to construct a model $p_\theta(y_1,y_2\vert x) = \mathcal N(\mu_\theta(x), \Sigma_\theta(x))$, where $\mu_\theta\in\mathbb R^2 = [\mu_{\theta,1},\mu_{\theta,2}]^\top$ and $\Sigma_\theta(x)\in\mathbb R^{2,2}$ with $\Sigma_{\theta,11}=\sigma_{\theta,1}^2$, $\Sigma_{\theta,22}=\sigma_{\theta,2}^2$ and $\Sigma_{\theta,12}=\Sigma_{\theta,21}=\sigma_{\theta,1}\sigma_{\theta,2}\rho_\theta$, where $\rho_\theta$ is the predicted correlation coefficient. This approach, however, is not the most practical, especially for distributions that are not Gaussian. That is why it can be desirable to decompose the joint distribution as $p_\theta(y_1,y_2\vert x)=p_\theta(y_2\vert x,y_1)\cdot p_\theta(y_1\vert x)$. This, however, requires the model to be able to accept an optional input, and behave differently accordingly. In~\cite{Durasov24a}, it is proposed to model $p_\theta(y_1\vert x)$ as $p_\theta(y_1\vert x, y_0)$, where $y_0$ is an uninformative constant specified as an hyperparameter. This way, the inference function of the model has to be modified only slightly, as show in Alg.~\ref{alg:original} and \ref{alg:zigzag}. 

\begin{minipage}{0.48\textwidth}
\begin{algorithm}[H]
\caption{Original Neural Network inference}
\label{alg:original}
\begin{algorithmic}[1]
\STATE \textbf{Input} input $x$
\STATE {\color{Gray} \# }
\STATE {\color{Gray} \# }
\STATE {\color{Gray} \# }
\STATE {\color{Gray} \# }
\STATE {\color{Gray} \# }
\STATE $x \gets \text{input}(x)$
\STATE $x \gets \text{hidden}(x)$
\STATE  \textbf{Return:} $\text{output}(x)$
\end{algorithmic}
\end{algorithm}
\end{minipage}
\hfill
\begin{minipage}{0.48\textwidth}
\begin{algorithm}[H]
\caption{Modified inference}
\label{alg:zigzag} 
\begin{algorithmic}[1]
\STATE \textbf{Input:} input $x$, constant $y_0$, feedback $y$ (optional)
\IF{$y$ is \text{None}}
    \STATE $B \gets x.\text{shape}[0]$
    \STATE $y \gets y_0 \cdot \text{ones}((B, \text{out\_features}))$
\ENDIF
\STATE $x \gets \text{concatenate}([x, y], \text{dim}=1)$
\STATE $x \gets \text{input}(x)$
\STATE $x \gets \text{hidden}(x)$
\STATE  \textbf{Return:} $\text{output}(x)$
\end{algorithmic}
\end{algorithm}
\end{minipage}

Having to fine tune the value of $y_0$ can be avoided by modifying the network by setting all the weight connected to the feedback neurons to 0, as shown in Alg~\ref{alg:new algo}.
\begin{algorithm}
\caption{Modified inference, no constant}
\label{alg:new algo}
\begin{algorithmic}[1]
\STATE \textbf{Input:} input $x$, feedback $y$ (optional)
\STATE $\text{feedback} \gets \text{False}$
\IF{$y$ is \text{None}}
    \STATE {\color{Gray} \# drop weights connected to $y$}
    \STATE $\text{feedback} \gets \text{True}$
    \STATE $\text{original\_weight} \gets \text{copy(input\_layer.weight))}$
    \STATE $\text{input\_layer.weight} \gets  \text{drop\_weights(input\_layer.weight)}$ 
\ENDIF
\STATE $x \gets \text{concatenate}([x, y], \text{dim}=1)$
\STATE $x \gets \text{input}(x)$
\STATE $x \gets \text{hidden}(x)$
\STATE $x\gets \text{output}(x)$

\IF{\text{feedback}}
\STATE {\color{Gray} \# restore original weights}
    \STATE$ \text{input\_layer.weight} \gets \text{original\_weight}$
    \ENDIF
\STATE  \textbf{Return:} $x$
\end{algorithmic}
\end{algorithm}

\section{Details of the physical experiments}
\subsection{Wind tunnel experiments}

\begin{figure}[!hbt]
    \centering
    \input{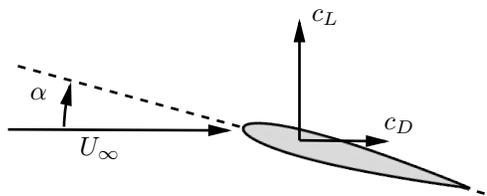}
    \caption{Schematics of the wind tunnel experiment.}
    \label{fig:schema wind tunnel}
\end{figure}

It is common practice in aerospace engineering to express the performances of an airfoil, i.e. a bidimensional section of a wing, in terms of its lift and drag coefficients, $c_L$ and $c_D$, defined as the ratio of the lift and drag forces (per unit length) and the dynamic pressure forces $1/2\rho c U_\infty^2 $, where $\rho$ is the fluid density and $c$ the profile chord. These forces can be measured by placing a maquette in a wind tunnel, where the flow speed can be controlled with precision, see Fig.~\ref{fig:schema wind tunnel} for a sketch of the main relevant quantities. The output of the force balances, once normalized, looks like Fig.~\ref{fig:dataset 250}. In black and white are indicated the points used for training, chosen independently placing a uniform distribution over the time series.
\begin{figure}[!hbt]
    \centering
    \includegraphics[width=0.6\linewidth]{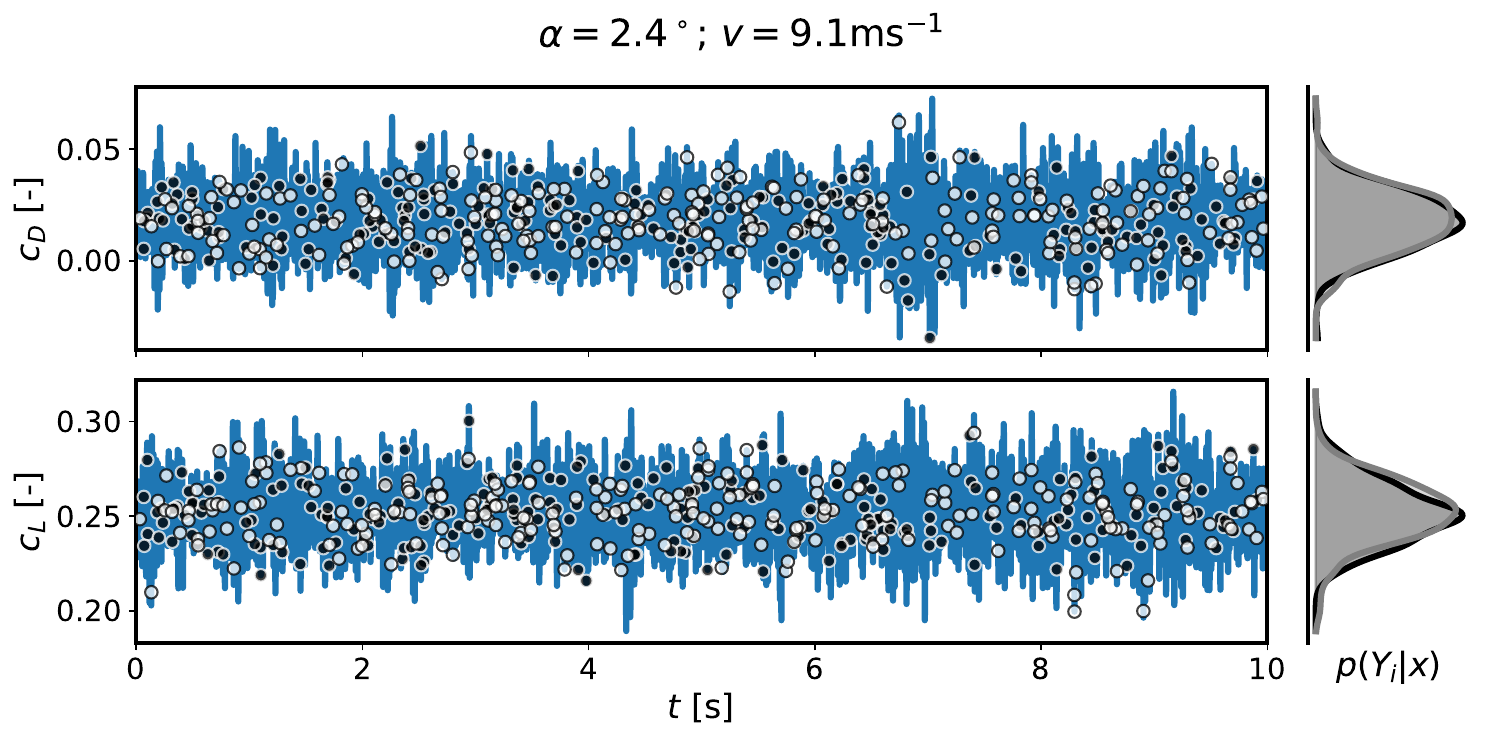}
    \caption{Normalized output of the force balances of the wind tunnel experiments. In black and white the samples used for modeling, which are chosen randomly with a uniform distribution over the set of all time samples.}
    \label{fig:dataset 250}
\end{figure}

\subsection{Drone noise}
The increased presence of small unmanned drones in daily life has revived the interest of the aeroacoustics community for rotor noise. The acoustic signature of a small rotor is mainly due to two effects: the random interactions of turbulent eddies with the blades, which gives rise to a broadband noise signature, and the rotation of the blades themselves, which produces sharp tones. The tones appear at the so-called blade passing frequency (BPF), i.e. the rotation speed of the motor $\Omega$ times the number of blades $n$, and all integer multiples of this fundamental frequency, called harmonics. Both these components are visible on the spectrum in Fig.~\ref{fig:spectrum}. Note that the extra peaks, not evidenced by vertical lines, are harmonics of the noise due to small imbalances of the fan system, and are thus not of aerodynamic nature. The tonal component of noise is the most annoying for the human ears, but is also hard to predict with precision because it depends on the unsteady pressure distribution on the blades, which require extremely costly numerical simulations to be captured efficiently.
\begin{figure}[!hbt]
    \centering
    \includegraphics[width=0.7\linewidth]{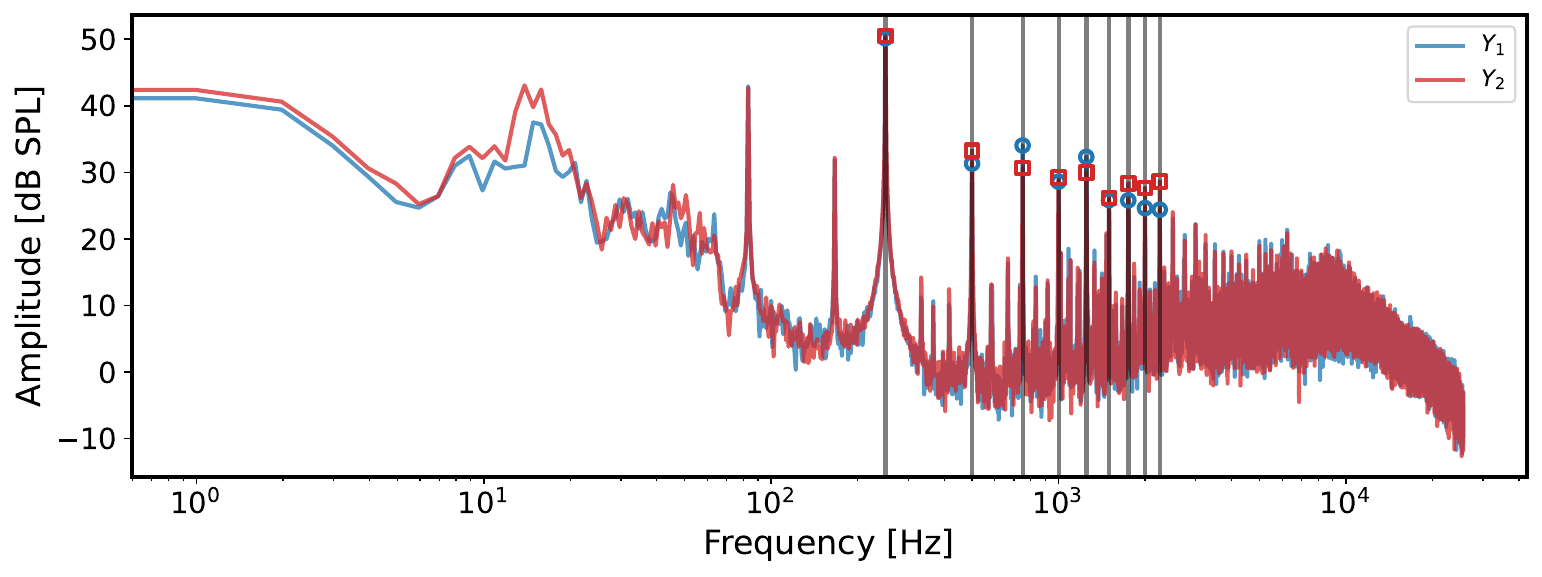}
    \caption{Spectrum captured at a microphone placed on the rotor-disk plane ($\vartheta=0^\circ$), for a three-bladed rotor spinning at $\Omega=3000$ rpm.}
    \label{fig:spectrum}
\end{figure}
The data has been collected in the anecho\"ic room of ISAE-SUPAERO, and are available on the Dataverse, along with scripts to perform the data processing. For the present study, the starting point are again the time series of the farfield sound pressure. The signal is split into two sub-sections, both of which are processed using the fast Fourier transform, averaging the results of 8 windowing sections, with no overlap, using the Hanning window. It is fair to consider the two halves as independent realisations of the same data distribution, by considering the time signal to be an ergodic random process. 

\end{document}